%% file: crowdstat-COLT-new-withnames.tex
\documentclass[11pt]{article}
\usepackage{amsmath,amssymb,amsthm,amsfonts,latexsym,bbm,xspace,graphicx,float,mathtools}
\usepackage[backref,colorlinks,citecolor=blue,bookmarks=true]{hyperref}
\usepackage[letterpaper,margin=1in]{geometry}
\usepackage{color}
\usepackage{algorithm}
\usepackage{algorithmic}
\usepackage{times}

\title{Optimum Statistical Estimation with Strategic Data Sources}

\author{Yang Cai\\
Computer Science, McGill University\\
{\tt cai@cs.mcgill.ca}
\and
Constantinos Daskalakis\\
EECS, MIT\\
{\tt costis@mit.edu}
\and
{Christos Papadimitriou}\\
Computer Science, U.C. Berkeley\\
{\tt christos@cs.berkeley.edu}
}

\newtheorem{theorem}{Theorem}

\newtheorem{remark}{Remark}
\newtheorem{definition}{Definition}

\newenvironment{prevproof}[2]{\noindent {\bf {Proof of {#1}~\ref{#2}:}}}{$\blacksquare$\vskip \belowdisplayskip}

\def\reals{\mathbb{R}}

\def\D{\mathcal{D}}
\def\W{\mathcal{W}}

\def\H{\mathcal{H}}

\def\E{\mathbb{E}}

\def\hf{\hat{f}}

\usepackage{color}
\definecolor{Red}{rgb}{1,0,0}

\begin{document}
\maketitle

\begin{abstract}
We propose an optimum mechanism for providing monetary incentives to the data sources of a statistical estimator such as linear regression, so that high quality data is provided at low cost, in the sense that the weighted sum of payments and estimation error is minimized.  The mechanism applies to a broad range of estimators, including linear and polynomial regression, kernel regression, and, under some additional assumptions, ridge regression. It also generalizes to several objectives, including minimizing estimation error subject to budget constraints. Besides our concrete results for regression problems, we contribute a mechanism design framework through which to design and analyze statistical estimators whose examples are supplied by workers with cost for labeling said examples.
\end{abstract}

\section{Introduction}\label{sec:intro}

Statistical estimation, from data, of the parameters of a model of reality, is the spirit, essence, and workhorse of modern science and business.  In today's complex world of science and industry, data for an estimator deployed by a particular research group or enterprise is often provided by other entities; furthermore, the quality of the data is crucial for the accuracy of the estimator, and can vary widely.  It is reasonable to assume that, with appropriate effort and cost, the providers of the data can improve the quality of their data --- but of course they may lack incentive to do so.   Crowdsourcing \cite{DRH} can be seen as a popular and widespread instantiation of the phenomenon.

The situation is not unlike {\em Mechanism Design,} a well developed field in Mathematical Economics~\cite{AGT}, and indeed related problems in connection to crowdsourcing have been recently treated within this framework; see the references in the next subsection.  In Mechanism Design, in order for the interaction of the designer with several rational strategic  agents to be as beneficial as possible, a game between the agents is created in which the pursuit by the agents of individual advantage leads to the optimum outcome for the designer.  Perhaps the archetypical Mechanism Design problem, solved in Myerson's celebrated work~\cite{Myerson}, is how to auction an item to a number of agents, whose values for the item are unknown but drawn from known prior distributions.  Myerson's is a powerful, sophisticated, and clean result, and has had tremendous impact---a useful ideal to keep in mind when venturing into new areas.

Coming back to estimation, suppose that a Statistician has an algorithm which, given appropriate data points $X=\{(x_i,y_i)\}$, approximates an unknown function $f$ that captures an important aspect of reality, and can be usefully employed in prediction.  For concreteness, let us say that the unknown function is linear and the context is one-dimensional linear regression (noting that our results are far more general).  The quality of the output of the algorithm, $\hat f$, depends crucially on the quality of the data $X$.  If the data is poor, the Statistician's prediction $\hat f(x^*)$---where $x^*$ is the unknown test point where a prediction will have to be made---will be off the mark.  Assume that the loss suffered is $(f(x^*) -\hat f(x^*))^2$.  How can the Statistician incentivize her data sources --- suppose for simplicity that each of them provides one of the $(x_i,y_i)$ points, for an $x_i$ selected by the Statistician --- to work hard so as to supply high-quality data? 

Assume that each data provider is a ``worker'' who, by expending effort $e$, achieves a level of quality of the data; the larger $e$, the higher the quality.  To fix ideas, assume that, by expending effort~$e$, the $i$th worker is able to sample a distribution for his target data $y_i$ that is centered at the true value $f(x_i)$ and has variance $\sigma_i(e)^2$.  That is, the variance of the error of the datapoint is our measure of the worker's quality (rather, lack thereof).  Of course there are many workers, and they may have different functions $\sigma_i(e)$.  {\em We assume that these functions are common knowledge} to the participants, including of course the Statistician.\footnote{In fact, as it will become clear, each worker needs to know this information only to decide whether to participate in the mechanism, and not to decide on his optimum effort level.}  The question we explore is this:   {\em Can the Statistician create a mechanism---or protocol, or contract---which incentivizes through appropriate {\em payments} the workers to supply high-quality data?}   We assume, naturally enough, that each worker acts so as to minimize the expectation of effort exerted minus payment received. (By rescaling, we can assume that units of loss by the Statistician, effort by a worker, and currency all coincide).  Naturally, this expectation must be nonnegative, because otherwise the worker will refuse to work.

The Statistician, of course, does not see the effort exerted by the $i$th worker, but only the result $y_i$ of this effort.  Evidently, the payment in the protocol, or contract, being designed should be determined by comparing each supplied $y_i$ to the true value $f(x_i)$.  But of course this is impossible  {\em because we will never learn $f(\cdot)$}.  The payments must somehow depend on what we know, which is the data, the algorithm, the quality functions of the workers, and the prior of the test point $x^*$, but importantly not $f$.

And what should such a mechanism aspire to achieve?  Let us define an ideal performance which, a priori, may seem unreasonably optimistic.  The quality of our estimation is going to depend on the following:
\begin{itemize}
\item the set $W$ of workers that we select from the available workers;
\item the set $\{x_i: i=1,\ldots,|W|\}$ of the regression points we choose to assign to each of them;
\item on the effort $e_i$ each worker decides to invest (which is out of our control).
\end{itemize}

Once these  have been determined, then we also know the variance $\sigma_i(e_i)^2$ of each data point, and the expectation of the loss to the statistician is a known function $L(W,\{x_i\},\{e_i\})$.  Define OPT to be the minimum, over all possible $W,\{x_i\},\{e_i\}$, of the quantity $L(W,\{x_i\},\{e_i\})+\sum_{i\in W} e_i$.\footnote{We can handle much more than this objective, as discussed later in this section.}  That is, OPT is the optimum of loss plus effort, what is known in Economics as the social optimum (best possible sum of costs by all participants).  But OPT is also a lower bound on the total cost (loss plus payments) to the Statistician:  To achieve total cost OPT, the Statistician must convince each worker to supply the precise optimum effort, and to do so {\em at zero surplus}, that is, by being paid the smallest possible amount for this effort (recall that the worker will refuse to work if payment minus effort has negative expectation). But this is hard to achieve, given that the efforts exerted by the workers are not observable and out of the Statistician's control.

So, {\em is there such a mechanism?}  In the light of the difficulties outlined in the last two paragraphs, the situation may seem quite hopeless.

Very surprisingly, we show that such a mechanism does exist not only for linear regression, but also for a broad class of estimators satisfying a simple and intuitive condition, and includes, for example, {polynomial regression, finite dimensional kernel regression and many other linear estimators}.  The solution engages workers in a game not with the Statistician, but {\em with each other.}  The game is constructed in such a way that, in the end, the optimum outcome is achieved, in that loss plus payments is minimized, and all ``surplus'' of the workers is extracted (that is, no worker is paid more than his work).  

One important consideration is what we mean above by ``in the end,'' that is, what are our assumptions on how workers will behave?  {\em What is our solution concept?}  After all, solution concepts are known to be delicate and fragile, the subject of endless discussions and controversy in Game Theory.  It turns out that our design can afford a solution concept that is extremely robust and uncontroversial: Each worker's decision is a {\em unique dominant strategy,} that is, it is the unique action that optimizes the worker's objective (payment minus effort) {\em no matter what anybody else does}.   Of course, this restricts severely the design space of mechanisms that we can employ, and despite this restriction we can still still design a mechanism that attains the ideal total cost OPT.

One key idea of the mechanism is that, in the game as defined, the payment of each worker depends on the data supplied by {\em the other workers}.  To some this may seem unreasonable, while to others it may seem a bit unsurprising in view of the VCG mechanism, among others, which has a similar structure \cite{Vickrey, Clarke, Groves}.  That this maneuver works in this instance is an entirely new phenomenon related to statistical estimation, as opposed to mechanism design, and comes with a different mathematical justification, even though in a sense it does have the same roots as VCG:  Our design essentially creates a {\em race for accuracy}, in which workers compete knowing that they will fare badly if left behind.

Our mechanism applies to all statistical estimators that satisfy a certain intuitive property:  The expectation of the estimator's loss depends only on the data points $\{x_i\}$ and the distribution of the test point $x^*$, and depends on the $y_i$'s {\em only through their variances}. There are two variants of the problem, depending on how the $x_i$'s are determined.  So far, we have assumed that the $x_i$'s {should be optimally chosen  by the algorithm/mechanism, and the estimators we can handle already include, e.g., linear regression, polynomial regression, finite-dimensional kernel regression, and several other linear estimators.}  In a model where the $x_i$'s are fixed and given in advance (but the algorithm should still optimally assign them to workers) the situation is even more favorable:  our technique applies to an even broader class of estimators, including {{\em ridge regression}} \cite{Tommi}. {This is discussed in Section~\ref{sec:discussion}.} In the same section, we also discuss several extensions of our results. Besides a weighted combination of loss plus payments, we can accommodate objectives that aim at minimizing loss subject to budget constraints, or loss plus a function of the payments, and can also go beyond mean square error to other loss functions.

\subsection{Related Work} \label{sec:related}
For background in statistical analysis and estimation see, e.g.,~\cite{HTF}, and for mechanism design see, e.g.,~\cite{AGT}. In the past few years there have been several papers treating crowdsourcing in a framework that is at least superficially similar to ours.  \cite{HC} gauge through experiments the elasticity of effort under pay in crowdsourcing, while \cite{HSW} use learning algorithms to find the optimum crowdsourcing contract, and \cite{RWM} add experts to the crowd so non-experts perform better.  Scheduling mechanisms are used to manipulate the time behavior of the crowd \cite{RBC,NDG}, whereas \cite{GNS} use incentives to match workers according to their specialization.  Several papers address strategies to optimize performance keeping within budget \cite{SKa, SM,CLZ}, while in \cite{BKS} the online task assignment problem is treated as a multi-armed bandit under a budget.  The optimal design of non-monetary ``prestige'' rewards to optimally incentivize the participants is addressed in \cite{ISS}, and in \cite{NixK12,IoannidisL13,SKr} privacy concerns in data gathering are treated through incentives; also in \cite{SKr2} regret minimization is used in a crowdsourcing context.  

A little closer to our framework, Mechanism Design has been used in \cite{DV,CHS} to analyze crowdsourcing {\em contests}---in which all workers submit their work and one is selected to be paid---as all-pay auctions.  These mechanisms also use monetary incentives to create a competition between workers to enhance performance, but their nature and the problems they address are quite different from ours---for example, they are by design {not individually rational.} {In~\cite{DekelFP08,MeirPR08,MeirPR09}, mechanism design is used for regression and classification with strategic data providers. In contrast to our setting, the agents are not  interested in being paid for producing data for the learning task and have no cost for exerting effort. Instead, they control a subset of the data for which they already have the correct labels, and want to bias the outcome of the learning process to perform well on their subset of the data. \cite{BaconPCRKS12} and follow-up papers design proper scoring rules for principal-agent problems where agents are asked to predict an outcome, but can also influence the outcome by exerting different levels of effort. 
In contrast to our setting, agents only care about rewards and have no cost for effort. More importantly, the designer only cares about how to incentivize the agents to exert maximum effort, and not how much money is paid to the agents. {Moreover, they use the final outcome to decide rewards (this is common in scoring rules), while we never learn $f(x^*)$.} \cite{DaguptaG13} design mechanisms for crowdsourced binary labeling, where a collection of agents are asked to provide binary labels to a collection of tasks, and the designer wants to incentivize them via monetary rewards to exert maximum effort for each task. The agents' strategy consists on  how much effort to exert {for each task  they are allocated} (where increased effort increases their cost but also the probability of correct labeling), and whether to truthfully report the labels {of their allocated tasks}. Besides being for a different learning task, the biggest difference to our paper lies in the fact that the designer does not care about how much money is paid to the agents, or the tradeoff between rewards and accuracy. \cite{PeerPrediction} use scoring rules to incentivize agents whose private signals are correlated to reveal their private signals. In comparison to our work, the agents' decision is whether to report their true private signal or some other signal, but not how much effort to exert to get better quality signals. Finally,~\cite{FSW} look at a problem close to ours except they assume that each worker cannot affect their quality (it is sampled from a prior distribution), but only decide whether to participate in the mechanism; they also only consider the simple unbiased estimator for constant functions from $\reals$ to $\reals$.}

\paragraph{Relation to Optimal Contract Theory/Principal-Agent Problems.} Our mechanisms achieve in a unique dominant strategy equilibrium objective value that an omnipotent dictator who could dictate how much effort each worker exerts would achieve. This is a surprising result that is morally related to results in optimal contract theory; see, e.g., Chapter 14 of~\cite{MasCollel}. The question is how to design a contract with a worker whose effort $e$ affects your revenue $\pi$ through some conditional density $f(\pi|e)$. The contract is a mapping $w(\cdot)$ from your revenue $\pi$ and the worker's effort (if it is observable) to the worker's income $w(\pi)$. So your utility is $\pi-w(\pi)$, while the worker's utility is $w(\pi)-g(e)$, for some function $g(\cdot)$. The problem is straightforward to formulate when the worker's effort is observable. The surprising result, which is reminiscent of our result, is that under appropriate conditions one can design a contract $w(\cdot)$ resulting in the same utility when the worker's effort is {\em not observable} as if it were observable. While the two results are related, the setting and techniques are different. Moreover, there are important qualitative differences in the results. The aforementioned result applies to a single worker, while ours applies to any number. Also, the worker's payment, even when his/her effort is not observable, is still a function of the revenue $\pi$, which is exactly observable. Instead, in our setting we do not observe what mean squared error we achieve as a result of the efforts exerted by the workers. The use of convexity in the payments, that we employ, has been employed in the design of non-linear contracts~\cite{FP}. (Many thanks to Preston McAfee for bringing this paper to our attention.)

\paragraph{Comparison to VCG.} The guarantees of our mechanism are also reminiscent to those of the celebrated VCG mechanism~\cite{Vickrey, Clarke, Groves}. VCG optimizes social welfare  in a dominant strategy equilibrium, achieving as high a welfare as that of an omniscient algorithm (who knows the bidders' valuations and can therefore compute the exactly optimal outcome). While the guarantees of our mechanism are similar to those of VCG, the technical reasons underlying the two results are different. Importantly, the VCG mechanism may allow multiple equilibria with worse guarantees. In contrast, our mechanism achieves optimality in a unique dominant strategy equilibrium. So, in particular, there are no bad equilibria. Moreover, VCG is very sensitive to the computational complexity of the underlying algorithmic problem. If the algorithmic problem is intractable, then so is running VCG, and the VCG mechanism is known to fail if the algorithmic problem can only be approximated. In contrast, our proposed mechanism is approximation preserving, as noted in Section~\ref{sec:complexity}.

%



\section{Estimation with Strategic Workers: the Model} \label{sec:prelim}

In this section, we introduce the statistical estimation task that we solve in the next section. We start with some standard definitions.

\begin{definition}[Estimator] \label{def:estimator}
Let $\H$ be a family of functions $f: \D \rightarrow \reals$, where  $\D \subseteq \reals^n$. An estimator for $\H$ takes as input a collection $(x_i,y_i)_{i=1}^k$ of examples $(x_i,y_i) \in \D \times \reals$ and produces an {\em estimated function} $\hf_{(x_i,y_i)_{i=1}^k} \in \H$.
\end{definition}

For example, $\H$ may be the class of linear functions from $\reals^n$ to $\reals$, in which case the estimator could be linear regression. If the input $(x_i,y_i)_{i=1}^k$ to the estimator is clear from context, we may omit it from the subscript of $\hf$. Sometimes we use $(\vec{x},\vec{y})$ as a shorthand for $(x_i,y_i)_{i=1}^k$, and use $\hf_{(\vec{x},\vec{y})}$ to denote the estimator. We may also use the shorthand $\hf_{-j}$ or $\hf_{(\vec{x},\vec{y})_{-j}}$ for the output of the estimator when given all examples except $(x_j, y_j)$; i.e. $\hf_{-j} \triangleq \hf_{(x_i,y_i)_{i \in \{1,\ldots,k\} \setminus \{j\}}} \triangleq \hf_{(\vec{x},\vec{y})_{-j}}$. This is assuming that our estimator is well-defined with one example omitted. Whenever we use this notation we assume that our estimator satisfies this property. We call $\hf_{-j}$ the {\em estimator $\hf$ with one example less}. 

\medskip  In estimation it is usually assumed that the examples are readily available.  Here we study the scenario in which we choose a collection $x_1,\ldots,x_k \in \D$ of points and assign them to experts, who then return estimates of the function $f$ at those points.  How good will these estimates be?  In this paper we assume that the experts, or {\em workers},  are {\em strategic;}  for example, they will put no effort into producing good estimates of the function value at their given point, unless they are provided monetary incentives to exert such effort. We capture the behavior of such strategic experts in the following definition. 

\begin{definition}[Strategic Worker]\label{def:worker}
Let $f \in \H$ be as in Definition~\ref{def:estimator}. A {\em worker for $f$} is a strategic agent who, given some query $x \in \D$, will decide how much effort $e \in {\cal E} \subseteq \reals_+$ to exert in order to produce an estimate $y(e)$ of $f(x)$.\footnote{Note that we have omitted the dependence of $y$ on $x$ in our notation to ease notation.} The worker:
\begin{itemize}
\item is characterized by some known strictly decreasing convex function $\sigma: {\cal E} \rightarrow \reals_+$ such that, whenever effort $e$ is exerted, the estimate produced satisfies: $$y(e) = f(x) + \epsilon,$$ where $\epsilon$ is distributed according to some (potentially unknown) distribution with mean $0$ and variance $\sigma(e)^2$.
 
\item aims to minimize the amount of exerted effort to produce the estimate of $f(x)$, unless provided monetary incentives to do otherwise; in particular, if the worker is promised a payment function $p: \D \times \reals \rightarrow \reals$ that assigns to each pair $(x,y)$ of a query point $x$ and estimate $y$ a dollar amount $p(x,y)$, then the worker will choose to exert an amount of effort 
$$e^* \in \arg \max_{e \in {\cal E}} \E[p(x,y(e))] - e,$$
where the expectation is taken with respect to the randomness in $y$ and the randomness in the payment function, if any.\footnote{Note again that we have omitted the dependence of $e^*$ on $x$ and the payment function $p(\cdot)$ in our notation.}
\end{itemize}
\end{definition}

The following definition formulates the problem of estimating an unknown function $f \in \H$, when one's only access to $f$ is through strategic workers for $f$. 

\begin{definition}[Estimation with Strategic Workers (ESW)] \label{def: estimation with strategic workers problem}
Suppose that we are given:
\begin{itemize}
\item an estimator $\hf$ for a family of functions $\H$, as in Definition~\ref{def:estimator}; 
\item access to a set $\W$ of strategic workers for some {\em unknown} function $f \in \H$, as in Definition~\ref{def:worker}, where each worker $i \in \W$ has a known function $\sigma_i$ mapping effort to accuracy; we also assume that all workers' estimations are independent;
\item a distribution $F$ over $\D$ (the distribution of the test point $x^*\in \D$).
\end{itemize}
Our goal is to:
\begin{enumerate}
\item choose some subset $\W' \subseteq \W$ of workers \label{decision:which experts to recruit?}
\item provide an input $x_i$ to each worker $i \in \W'$, requesting an estimate $y_i$ of $f(x_i)$ from $i$ \label{decision: what inputs to give them?}
\item commit to a payment function $p_i$ to each $i \in \W'$, where $p_i$ is a  
(potentially) randomized mapping $p_i:~(x_i, y_i)_{i \in \W'}~\mapsto~\reals$, which may depend not only on the estimate produced by worker $i$ but also the estimates produced by the other workers.\label{decision: how to pay them?}
\end{enumerate}
Subject to our decisions in~\ref{decision:which experts to recruit?}, \ref{decision: what inputs to give them?} and~\ref{decision: how to pay them?}, we are looking to minimize a weighted average of the mean-square error of our estimation $\hf$ and the expected payments made to the workers, namely: 
\begin{align}
\E_{x^*,\vec{y}(\vec{e}^*)}\left[\left(\hf_{(\vec{x},\vec{y}(\vec{e}^*))}(x^*)-f(x^*)\right)^2 + \eta \cdot \sum_{i \in \W'} p_i\left((x_j,y_j(e^*_j))_{j \in \W'}\right) \right], \label{eq:statistician's objective}
\end{align} for some $\eta >0$, where the expectation is taken with respect to all the randomness in the setting: the randomness in $x^* \sim F$, the randomness in the outputs $\vec{y}(\vec{e}^*) \triangleq \{y_i(e_i^*)\}_{i \in \W'}$ produced by the workers, and the randomness in the payment functions. For~\eqref{eq:statistician's objective} to be a well-defined objective, we need to be able to predict the efforts $(e^*_i)_{i \in \W'}$ that the workers of our selected set will exert given our decisions for~\ref{decision:which experts to recruit?}, \ref{decision: what inputs to give them?} and~\ref{decision: how to pay them?}. We discuss how this can be achieved in Section~\ref{sec:incentives} below. At the very least, our prediction needs to satisfy the individual rationality constraint of Definition~\ref{def:IR}, i.e. that the expected payment to each worker $i$ is at least as large as $e^*_i$, otherwise the worker would not participate.
%
%
\end{definition}

\begin{remark}[Variants]
\begin{enumerate}
\item In Definition~\ref{def: estimation with strategic workers problem}, we assume that the designer optimizes over the selection of  the points $x_i$.  Alternatively, it could be that each expert $i$ comes with a point $x_i$, which is ``his expertise,'' or that the $x_i$'s are predetermined but the designer can still decide how to assign them to experts.  In Section~\ref{sec:discussion}, we briefly discuss this variant of our problem, which yields a much richer class of estimators for which our result, described in the next section, applies.
\item We also assumed that the objective is to minimize a weighted average of mean-square error and expected payments. While we stick to this objective for the development of our mechanism in the next section, our techniques go through with minimal modifications to much more general objectives. For example, we can accommodate the problem of minimizing mean-square error subject to a budget constraint, or minimizing the sum of mean-square error and an arbitrary increasing function of payments. We discuss these extensions in Section~\ref{sec:discussion}.
\item Finally, our technique is general enough to even eliminate the use of mean-square error from the objective, an extension that we also discuss in Section~\ref{sec:discussion}.
\end{enumerate}
\end{remark}

\subsection{Incentives} \label{sec:incentives}
We have already noted inside Definition~\ref{def: estimation with strategic workers problem} that for~\eqref{eq:statistician's objective} to be a well-defined objective, we need to be able to predict the efforts $(e^*_i)_{i \in \W'}$ that the workers will exert as a result of our solution to ESW.   It is important to note that the form of the payment functions (Decision~\ref{decision: how to pay them?} in Definition~\ref{def: estimation with strategic workers problem}) couples the decision of each worker $i$ about the amount of effort he exerts with the amounts of effort the other workers exert (since these influence $\vec{y}_{-i}$), which themselves depend on $y_i$ and hence the effort that worker $i$ exerts. This cyclical dependence is familiar in Game Theory. Indeed, any solution to ESW---comprising a subset $\W'$ of workers, queries $(x_i)_{i \in \W'}$ to them, and payment commitments $(p_i)_{i \in \W'}$---induces a {\em game} among the workers in $\W'$; the effort levels $(e^*_j)_{j \in \W'}$ eventually chosen by the workers are the outcome of their strategic interaction in this game. Therefore, to be able to evaluate~\eqref{eq:statistician's objective}, we need to predict how agents will behave in this game. 

There are numerous {\em solution concepts} in Game Theory, whose goal is to close into the possible behavior of  rational players in a game. The prominent ones are Nash equilibrium and its several refinements. However, Nash equilibria may be randomized and, most problematically, they may be multiple, which would result into equilibrium selection issues in our setting. To avoid such issues and guarantee robustness of our solutions, we insist on the most compelling and uncontroversial solution concept, namely that of
%
{\em Unique Dominant Strategy Equilibrium,} defined next.  
\begin{definition}[Unique Dominant Strategy Equilibrium]\label{def:DSE}
A solution to ESW---comprising a subset $\W'$ of workers, queries $(x_i)_{i \in \W'}$, and payment commitments $(p_i)_{i \in \W'}$---induces a {\em unique dominant strategy equilibrium} $(e^*_i)_{i \in \W'}$ iff, for all $i \in \W'$ and all $(e_j)_{j \in \W'}$:
$$\E\left[p_i\left((x_i,y_i(e^*_i)), (x_j, y_j(e_j))_{j \in \W' \setminus \{i\}} \right)\right] - e^*_i\ge \E\left[p_i\left((x_j, y_j(e_j)\right)_{j \in \W' }\right]- e_i,$$
where the expectation is with respect to everything that is random, with equality only if $e_i=e^*_i$. In words, no matter what effort levels the other workers choose, the unique optimal effort level of every worker $i$ is $e^*_i$.
\end{definition}
Note that, if a unique dominant strategy equilibrium exists in a game it is fairly trivial for agents to decide what strategy to play. We will insist that our solution to ESW should induce a game that has a unique dominant strategy equilibrium. Note that, in general, it is very rare for a game to have such an outcome, and consequently this poses a significant constraint on our design. {While our main result (Theorem~\ref{thm:main theorem}) satisfies this constraint, it nevertheless does not sacrifice any objective value, fairing as well as it would without this constraint present; see Remark~\ref{def:disambiguation} below and the discussion following the statement of Theorem~\ref{thm:main theorem}.}

Finally, as we have already noted in Definition~\ref{def: estimation with strategic workers problem}, not all combinations of solutions to  ESW and predictions of worker behavior are realistic. Since the workers are assumed strategic and their participation is voluntary, they should not be making a loss when participating. This is captured by the following definition, adding an additional requirement to our solutions to ESW.

\begin{definition}[Individual Rationality]\label{def:IR}
Given a solution to ESW---comprising a subset $\W'$ of workers, queries $(x_i)_{i \in \W'}$, and payment commitments $(p_i)_{i \in \W'}$, a collection of efforts $(e^*_i)_{i \in \W'}$ satisfies {\em individual rationality} iff, for all workers $i \in \W'$,
$$ \E\left[p_i\left( \left(x_j,y_j(e^*_j)\right)_{j \in \W'}\right)\right] - e^*_i \ge 0.$$
\end{definition}

\begin{remark}[Robustness vs Optimality Non-Tradeoffs] \label{def:disambiguation}
We have chosen to restrict our attention to solutions to ESW that induce a game among workers with a unique dominant strategy equilibrium that satisfies individual rationality. The worry might be that such a strong requirement might sacrifice too much objective value. We will show that it does not, in a very strong sense. In particular, while this requirement is handicapping our  solutions to ESW, we still compare 
our solution's performance against any other solution evaluated at {\em the most favorable} for that solution collection $(e^*_i)_i$, with the minimum requirement that the $(e^*_i)_i$ satisfy the individual rationality constraint.
\end{remark}

\section{Optimal Estimation with Strategic Workers: the Mechanism} \label{sec:main result}

Our main contribution in this paper is to establish the existence of an optimal solution to ESW, which induces a unique dominant strategy equilibrium and which has very strong optimality guarantees, as discussed later in this section, for a broad class of estimators $\hf$, containing several familiar ones:
\begin{definition}\label{def:well behaved}
An estimator $\hf$ for $\H$, as in Definition~\ref{def:estimator}, is {\em well-behaved} iff there exists some function $g$ such that, for all distributions $F$ over $\D$, functions $f \in \H$, and vectors $\vec{x} \in \D^*$ and $\vec{\sigma} \in \reals_+^*$ (of the same dimension as $\vec{x}$):\footnote{We use the shorthand $\D^* \triangleq \bigcup_{i=1}^{\infty}\D^i$, and similarly for $\reals_+^*$.}
$$\E_{\vec{y}, x^*}\left[\left( \hf_{(\vec{x},\vec{y})}(x^*)-f(x^*) \right)^2\right] = g(\vec{x}, F, \vec{\sigma}),$$
where for the purposes of the expectation on the left hand side $x^* \sim F$ and, independently for all $i$, $y_i = f(x_i)+\epsilon_i$, where  each $\epsilon_i$ is sampled from an arbitrary distribution of mean $0$ and variance $\sigma_i^2$, (and when $\vec{x}$ is such that $\hf_{(\vec{x},\vec{y})}$ is well-defined).
\end{definition}
Note that several common estimators, such as linear regression, {polynomial regression, finite-dimensional kernel estimation}, are well-behaved according to our definition.\footnote{{For instance, if $\hat{f}$ is linear regression then $$g(\vec{x}, F, \vec{\sigma}) = \E_{x^* \sim F} \left[ [(x^*)^{\rm T}, 1] \cdot (X^{\rm T} X)^{-1} X^{\rm T} \cdot {\rm diag}(\vec{\sigma^2})\cdot X (X^{\rm T} X)^{-1} \cdot [(x^*)^{\rm T}, 1]^{\rm T}\right],$$
where $X$ is the matrix whose rows are $[x_i^{\rm T}, 1]$ and ${\rm diag}(\vec{\sigma^2})$ is the diagonal matrix whose $(i,i)$ entry is $\sigma_i^2$.}}

\medskip Our main result is the existence of an optimal algorithm for ESW, whenever $\hf$ is well-behaved, according to Definition~\ref{def:well behaved}, and well-defined with one example less.\footnote{An estimator is not well-defined with one example less if omitting one example from its input makes the output of the estimator undefined. For example, if $\hf$ is linear regression (for linear functions from $\reals^n$ to $\reals$) restricted to take as input exactly $n+1$ examples $(x_i,y_i)$, then it is not well-defined with one example less, since $n$ examples won't suffice to produce an estimate.} We note that the first condition is a sufficient condition and discuss how to relax it in Section~\ref{sec:discussion}. The second condition is necessary for interesting solutions to ESW, and we discuss how it can be removed by broadening the set of allowable payment functions in Section~\ref{sec:discussion}. Our algorithm optimally solves ESW in a rather strong sense, as captured by Properties~\ref{property: really good solution} and~\ref{property:full surplus} in the following theorem.

\begin{theorem}\label{thm:main theorem}
There exists an optimal algorithm for ESW for all well-behaved estimators $\hf$ that are well-defined with one example less. The algorithm:
\begin{enumerate}
\item produces a solution to ESW that induces a unique dominant strategy equilibrium that satisfies individual rationality; \label{property:IR and DST}
\item under the unique dominant strategy equilibrium the solution achieves objective value \eqref{eq:statistician's objective} that is optimal; in fact, the achieved objective value matches the following quantity:
\begin{align}
\min_{\W', (x_i, e_i)_{i \in \W'}} \left( \E_{x^*, \vec{y}}\left[\left(\hf_{(\vec{x},\vec{y})}(x^*)-f(x^*)\right)^2  \right] + \eta \cdot \sum_i e_i \right), \label{eq:stronger objective}
\end{align}
where for the purposes of the expectation we assume that, for all $i\in \W'$, $y_i = f(x_i)+\epsilon_i$, where $\epsilon_i$ is sampled according to worker $i$'s distribution when s/he exerts effort $e_i$ (which has mean $0$ and variance $\sigma_i(e_i)^2$);
\label{property: really good solution}
\item extracts optimal worker surplus; in particular, the expected utility of every worker is $0$ at the unique dominant strategy equilibrium. \label{property:full surplus}
\end{enumerate}
\end{theorem}
Notice that Quantity~\eqref{eq:stronger objective} clearly provides a lower bound to the objective value \eqref{eq:statistician's objective} of {\em any} solution to ESW (and not just those inducing a unique dominant strategy equilibrium), evaluated at {\em any} $(e^*_i)_{i \in \W'}$ that satisfies individual rationality. Indeed, by individual rationality for all workers combined, 
$$\E\left[\sum_{i \in \W'} p_i\left((x_j,y_j(e^*_j))_{j \in \W'}\right) \right] \ge \sum_i e^*_i.$$
Hence, $\eqref{eq:statistician's objective} \ge \eqref{eq:stronger objective}$. In fact, \eqref{eq:stronger objective} corresponds to the objective value that one would achieve, if one could dictate the effort level that each worker should exert and only paid workers exactly for the amount of effort they exerted and not a cent more. So, in fact, Property~\ref{property: really good solution} implies Property~\ref{property:full surplus} in our theorem above. What our theorem establishes is that there always exist solutions to ESW that induce a unique dominant strategy equilibrium satisfying individual rationality, and that these solutions achieve the same objective value that a dictator who could dictate the behavior of each worker would achieve. Even though we do not assume we have such power, we still achieve the same objective value that such a powerful dictator would.

\begin{prevproof}{Theorem}{thm:main theorem}
We design a solution to ESW whose unique dominant strategy equilibrium $\vec{e}^*$ satisfies individual rationality, and achieves objective value~\eqref{eq:statistician's objective} that equals Quantity~\eqref{eq:stronger objective}. We have already argued that if we do this, we immediately satisfy  Properties~\ref{property:IR and DST},~\ref{property: really good solution} and~\ref{property:full surplus} in the statement of the theorem.

We define our solution to ESW in terms of an arbitrary optimal solution $\W', (x_i, e_i)_{i \in \W'}$ to the minimization problem~\eqref{eq:stronger objective}. In terms of this solution:
\begin{itemize}
\item We choose the same set of workers $\W'$ and assign to each $i \in \W'$ the point $x_i$.

\item It remains to define our payment commitments to the workers. To each worker $i \in W'$, we commit to the payment:\footnote{For compactness, we denote by $(\vec{x},\vec{y})=(x_i,y_i)_{i \in \W'}$.}
\begin{align}
p_i((\vec{x},\vec{y})) = c_i - d_i \cdot \left(y_i - \hat{f}_{(\vec{x},\vec{y})_{-i}}(x_i) \right)^2, \label{eq:payment form}
\end{align}
for some $c_i, d_i$ to be chosen. Notice that the payment to worker $i$ depends also on the reports of the other workers. 
\end{itemize}
We now choose the constants $(c_i,d_i)_{i \in \W'}$ so that our solution induces a unique dominant strategy equilibrium $\vec{e}^*$ that satisfies individual rationality (with equality) and also $\vec{e}^* \equiv \vec{e}$, where $\vec{e}=(e_i)_{i\in\W'}$ is as in the solution to~\eqref{eq:stronger objective} that we have fixed. 
What is the expected payment to worker $i$ if the workers exert some arbitrary efforts $\vec{e}'$? Denoting $\vec{y}(\vec{e}')=(y_i(e'_i))_{i \in \W'}$, we have:
\begin{align*}&\mathbb{E}_{\vec{y}(\vec{e}')}[p_i((\vec{x},\vec{y}(\vec{e}')))]\\
&~~~=c_i - d_i \cdot \mathbb{E}_{\vec{y}(\vec{e}')}\left[\left(y_i-f(x_i)\right)^2 + \left(f(x_i)-\hat{f}_{(\vec{x},\vec{y})_{-i}}(x_i)\right)^2 - 2 \left(y_i - f(x_i)\right)\left(f(x_i)-\hat{f}_{(\vec{x},\vec{y})_{-i}}(x_i)\right)\right]\\
&~~~=c_i - d_i \cdot \left(\sigma_i(e'_i)^2 + g(\vec{x}_{-i}, \boldsymbol{1}_{x_i}, \vec{\sigma}_{-i}(\vec{e}_{-i}'))\right),
\end{align*}
where we used that our estimator $\hat{f}$ is well-behaved, according to Definition~\ref{def:well behaved}, and the independence of the estimation of worker $i$ and the other workers. We denote by $\vec{\sigma}_{-i}(\vec{e}_{-i}')=(\sigma_j(e_j'))_{j \neq i}$, and by $\boldsymbol{1}_{x_i}$ the distribution that samples $x_i$ with probability $1$. $g$ is a known function determined by the estimator according to Definition~\ref{def:well behaved}.

Since each worker $i \in \W'$ is assumed rational, aiming to maximize his expected payment minus exerted effort, if the other workers exert effort levels $\vec{e}'_{-i}$, worker $i$'s best response is found by solving the maximization problem:
\begin{align}
\max_{e_i'} \left(c_i - d_i \cdot \left(\sigma_i(e'_i)^2 + g(\vec{x}_{-i}, \boldsymbol{1}_{x_i}, \vec{\sigma}_{-i}(\vec{e}_{-i}'))\right) - e_i'\right). \label{eq:costas}
\end{align}
Taking derivative with respect to $e_i'$ and setting to $0$ gives the following condition for the optimum $e^*_i$:
$$2 d_i \sigma_i(e^*_i)\cdot \sigma_i'(e^*_i) + 1 =0.$$
In order to ensure that $e^*_i \equiv e_i$, where $e_i$ was the effort level computed by solving~\eqref{eq:stronger objective}, we set
$$d_i = {-1 \over 2 \sigma_i(e_i) \sigma'_i(e_i)}.$$
Given that $\sigma_i(\cdot)$ is convex decreasing, our setting of $d_i$ ensures that $e_i$ is the unique solution to~\eqref{eq:costas}.

So our choice of $(d_i)_{i\in \W'}$ has made sure that the unique dominant strategy equilibrium of the game among workers defined by our solution to ESW (regardless of the $c_i$'s) is $(e_i)_{i\in \W'}$. Now we set the $c_i$'s so that this equilibrium also satisfies individual rationality tightly. It suffices to choose, for each $i \in \W'$: 
$$c_i = d_i \cdot \left(\sigma_i(e_i)^2 + g(\vec{x}_{-i}, \boldsymbol{1}_{x_i}, \vec{\sigma}_{-i}(\vec{e}_{-i}))\right) + e_i.$$
This choice makes sure that the expected payment to each worker equals his effort. Hence, the unique dominant strategy equilibrium of the game among workers defined by our solution to ESW is $\vec{e}$ and it satisfies:
$$\E\left[\sum_{i \in \W'} p_i\left((x_j,y_j(e_j))_{j \in \W'}\right) \right] = \sum_i e_i.$$
Hence, the objective value \eqref{eq:statistician's objective} achieved by the unique dominant strategy equilibrium matches~\eqref{eq:stronger objective}.
This concludes the proof of the theorem.
\end{prevproof}

\input{complexity}

\input{discussion}

\bibliographystyle{alpha}

\input{crowdstat.bbl}
\input{related}
\end{document}

%% file: complexity.tex
\subsection{The Computational Complexity of Our Algorithms}\label{sec:complexity}

An important question is, of course, how the Statistician could arrive at the optimum mechanism in a computationally efficient manner. From the proof of Theorem~\ref{thm:main theorem}, it is clear that the major computational overhead is finding an optimal solution to the minimization problem~\eqref{eq:stronger objective}. All the other steps of our algorithm can be executed in polynomial time. In this section, we study the computational complexity of the minimization problem~\eqref{eq:stronger objective}. 

A closely related problem to ours, {\em optimal experiment design (ODE)}, has received lots of attention in the Statistics community~\cite{Kie59,KW59,KW60,Kie61,KW65,Kie74,Puk93}. In ODE, the goal is also to find an optimal set of regression vectors $\{x_{i}\}$ for estimation of their corresponding $\{f(x_i)\}$ such that a certain objective is optimized. Depending on the objective, different criteria of optimality have been proposed e.g. A-optimality, C-optimality, D-optimality etc.. However, unlike  our model, the estimation error of $f(x_i)$ for any regression vector $x_i$ in ODE is assumed to come from a fixed distribution, independent of the worker's effort level. Hence, our problem can be viewed as a generalization of ODE. In particular, it is straightforward to show that the C-optimal ODE problem is a special case of ours. Unfortunately, even this special case is NP-hard to solve exactly~\cite{CH12}. We propose two different approaches to address this computational intractability. 

\paragraph{Robustness to Approximation.} The first approach is to use an approximate solution to the minimization problem~\eqref{eq:stronger objective} in our mechanism. It is clear from the proof of Theorem~\ref{thm:main theorem} that our mechanism can set the unique dominant strategy equilibrium to be any feasible solution of the minimization problem~\eqref{eq:stronger objective}. Moreover, the objective value~\eqref{eq:stronger objective} of our mechanism at this unique dominant strategy equilibrium equals the value of the minimization problem~\eqref{eq:stronger objective} on that particular feasible solution. In short, our mechanism provides {\em an approximation preserving reduction} from the ESW problem to the minimization problem~\eqref{eq:stronger objective}. I.e. whenever an $\alpha$-factor approximation algorithm exists for~\eqref{eq:stronger objective}, our technique translates that (in a black-box manner) to an $\alpha$-factor approximation algorithm for ESW.

\paragraph{Optimal Assignment Problem.} We may also consider a special case of the minimization problem~\eqref{eq:stronger objective}, where the set of regression vectors are predetermined but not assigned to the workers. Formally, we restrict \eqref{eq:stronger objective} to the following minimization problem: 
\begin{align}
\min_{\W', (x_{i}:=s_{\pi(i)}, e_i)_{i \in \W'}} \left( \E_{x^*, \vec{y}}\left[\left(\hf_{(\vec{x},\vec{y})}(x^*)-f(x^*)\right)^2  \right] + \eta\cdot\sum_i e_i \right), \label{eq:fixed samples objective}
\end{align}
where $\vec{s}=(s_{1},\ldots,s_{k})$ is a fixed set of regression vectors and $\pi$ is a bijection from $W'$ to $[k]$. We show that if the well-behaved estimator $\hf$ is {\em separable},\footnote{Linear regression, polynomial regression and finite-dimensional kernel regression are all separable.} the minimization problem~\eqref{eq:fixed samples objective} is solvable in polynomial time via min-cost bipartite matching. 
\begin{definition}\label{def:separable}
A {\em well-behaved} estimator $\hf$ for $\H$, as in Definition~\ref{def:well behaved}, is {\em separable} iff there exists a function $h$, such that, for all distributions $F$ over $\D$, functions $f \in \H$, and vectors $\vec{x} \in \D^*$ and $\vec{\sigma} \in \reals_+^*$ (of the same dimension as $\vec{x}$): $$\E_{\vec{y}, x^*}\left[\left( \hf_{(\vec{x},\vec{y})}(x^*)-f(x^*) \right)^2\right]  = \sum_{i} h(x_{i},X,{F})\cdot \sigma_{i}^{2},$$ where $X$ is the set of all $x_{i}$'s, and for the purposes of the expectation on the left hand side $x^* \sim F$ and, independently for all $i$, $y_i = f(x_i)+\epsilon_i$, where  each $\epsilon_i$ is sampled from an arbitrary distribution of mean $0$ and variance $\sigma_i^2$, (and when $\vec{x}$ is such that $\hf_{(\vec{x},\vec{y})}$ is well-defined).
\end{definition}

\begin{theorem}\label{thm:fixed samples}
Given a set of regression vectors $\vec{s}=(s_{1},\ldots s_{k})$, and a {\em separable} estimator$\hf$ for $\H$, there is a polynomial time algorithm for finding an optimal solution to the minimization problem~\eqref{eq:fixed samples objective}.
\end{theorem}
\begin{proof}
We reduce the minimization problem to the min-cost bipartite matching problem. Let $S$ be the set of regression vectors in $\vec{s}$, $(W,T)$ be a complete bipartite graph, such that $W$ is the set of all workers and $T$ is the set of all regression vectors in $S$ plus $|W|-k$ dummy regression vectors. For any edge $(i,j)$ with $j\in [k]$, the cost $c_{i,j}$ is $\min_{e_{i}} h(s_{j},S,{F})\cdot\sigma_{i}(e_{i})^{2}+\eta\cdot e_{i}$. Since $\sigma_{i}(e_{i})$ is a convex function, this cost can be computed in polynomial time. For any edge incident to a dummy regression vector, the cost is $0$. 

Now consider~\eqref{eq:fixed samples objective}. Since $\hf$ is a separable estimator, the minimization problem can be simplified as \begin{align*}&\min_{\W', (s_{\pi(i)}, e_i)_{i \in \W'}} \left( \sum_{i}h(s_{\pi(i)},S,F)\cdot\sigma_{i}(e_{i})^{2}+ \eta\cdot\sum_i e_i \right)\\
=&\min_{\W',\pi}\sum_{i\in W'} \min_{e_{i}}h(s_{\pi(i)},S,F)\cdot\sigma_{i}(e_{i})^{2}+\eta\cdot e_{i}\\
=&\min_{\W',\pi}\sum_{i\in W'} c_{i,\pi_{i}}.\end{align*}

The formula above wants to find a subset of workers $W'$ and match them to the regression vectors in $S$, such that the total cost is minimized. Notice that matching any worker to a dummy regression vector has cost $0$, therefore finding an optimal solution to the formula above is equivalent as finding a min-cost bipartite matching in $(W,T)$. Given any min-cost bipartite matching $M$, we can construct an optimal solution to~\eqref{eq:fixed samples objective} by assigning regression vector $s_{j}$ to $j$'s partner $i$ in $M$. Since finding a min-cost bipartite matching takes only polynomial time,~\eqref{eq:fixed samples objective} can also be solved in polynomial time.\end{proof} 

%% file: discussion.tex
\subsection{Discussion and Extensions} \label{sec:discussion}

One of the conditions in Theorem~\ref{thm:main theorem} is that the estimator $\hf$ is well-defined with one example less. Without it, we cannot hope for any interesting solutions to ESW. As a trivial example, suppose that $\hf$ is the unbiased estimator of constant functions from $\reals$ to $\reals$, which takes as input one example $(x,y)$ and outputs $y$. In this case $\hf$ is not well-defined with one example less, and no interesting solutions to ESW exist, since the payments to a worker may only depend on his own report, and hence the worker will put the minimal effort in set $\cal E$, regardless of our payment, as long as our payment is at least that minimal effort. (Recall that, since we never learn the unknown function $f$ we cannot use it to penalize the worker.) In this example, our constraint that $\hf$ is well-defined with one example less effectively says that we need to use at least two workers for estimating constant functions from $\reals$ to~$\reals$. 

We point out that this requirement can be removed in settings where the test point $x^*$ and the function value $f(x^*)$ can be observed by the Statistician and the experts. We can then treat $(x^*,f(x^*))$ as an additional example with zero variance, and modify \eqref{eq:payment form} to include that example in $\hat{f}$.

Our condition that $\hf$ is well-behaved does not contain regularized estimators, such as ridge regression. As ridge regression is biased, its mean square error also includes a bias term that depends on $f$. This could create two potential problems for our mechanism. (i) As we do not know $f$, we can not possibly solve the minimization problem~\eqref{eq:stronger objective} even if we are given infinite computational power. (ii) This term will also appear in every worker's expected payment (if we use payments as in~\eqref{eq:payment form}), and, since $f$ is unknown, the workers can't evaluate their expected utilities exactly, and thus can't decide if it's beneficial for them to participate in the mechanism. We point out that, under the assumption that the $x_i$'s are fixed in advance (but the Statistician is still allowed to optimally assign them to workers) our mechanism can address both problems and accommodate ridge regression with only mild modifications. For problem (i), the mean square error of ridge regression can be separated into two parts -- the bias term and the variance term. The bias term depends on $f$ and the $x_{i}$'s but is independent of the $\sigma_{i}$'s. This means that no matter how the samples are assigned and how they are estimated, the bias term remains the same. So we only need to consider the variance term. Luckily the variance does not depend on $f$, thus we can still find the optimal solution of the minimization problem~\eqref{eq:stronger objective}. For problem (ii), we can modify our payments. Roughly, we can use any unbiased estimator $\tilde{f}$ and add another term $(\hf-\tilde{f})^2$ to the payment function (both estimators applied to the examples from the other workers). The extra term will cancel the bias in expectation and will allow the workers to reason about their optimal behavior even without knowing~$f$. Note that the new term will introduce some extra terms in the expected utility, but these only depend on known quantities such as the variances, and the worker will be able to reason about his optimal behavior.

Although our objective function is stated as the weighted sum of mean square error and the total payment in ESW, our mechanism can accommodate many variants of that objective function. In general, we can modify the minimization problem~\eqref{eq:stronger objective} to reflect the new objective function. We can then apply our same mechanism to enforce that the unique dominant strategy equilibrium is achieved at the optimal solution of the new minimization problem. Here we list a few variants of our objective function and show how to modify the minimization problem~\eqref{eq:stronger objective} accordingly.
\vspace{.1in}

\noindent{\bf Worker-specific scaling factor:} An easy generalization is to replace the single scaling factor for the total payment to worker-specific scaling factors in the objective. To accommodate this objective, we only need to change the scaling factor for the sum of efforts to worker-specific scaling factors in~\eqref{eq:stronger objective}.

\vspace{.1in}

\noindent{\bf Budget Constraints:} A possible alternative objective is minimizing the mean square error subject to a budget constraint for the total payment. For this objective, we update~\eqref{eq:stronger objective} to minimize the mean square error subject to the constraint that the sum of all workers' efforts is no greater than the budget. 

\vspace{.1in}

\noindent{\bf Replacing the total payment with any increasing function of the payments:} Another possible generalization of our objective is to replace the total payment with any increasing function of the payments $q(\vec{p})$. In this case, we update~\eqref{eq:stronger objective} to minimize the sum of mean square error plus $q(\vec{e})$.

\vspace{.1in}
\noindent{\bf Budget on any increasing function of the payments:} We can also impose budget constraints on any increasing function of the payments. For this objective, we modify~\eqref{eq:stronger objective} to minimize the mean square error subject to the constraint that the increasing function over the workers' efforts does not exceed the budgets.

\vspace{.1in}
\noindent{\bf Beyond Mean Square Error:} Our techniques can be generalized to accommodate general error functions beyond mean-square. If the function combined with the estimator are well-behaved, that is, it is independent of the function $f$, then our techniques go through.  We only need to modify~\eqref{eq:stronger objective} to minimize the weighted sum of the new error function and the total effort.